\documentclass[runningheads]{llncs}
\usepackage[T1]{fontenc}
\usepackage{graphicx, amsmath, amssymb, xspace, scalerel,enumitem}

\usepackage{amsthm}
\usepackage{thm-restate}

\providecommand{\A}{\mathcal{A}}

\providecommand{\nat}{\mathbb{N}}

\providecommand{\Strat}{\textit{Strat}}
\providecommand{\strat}{\Strat}

\providecommand{\pfin}{{P^{\text{fin}}}}
\providecommand{\ce}{{\between}}

\providecommand{\us}{\Gamma}
\providecommand{\them}{\Lambda}

\providecommand{\aco}{approximated consequence operator\xspace}
\providecommand{\ado}{\aco}

\providecommand{\axiom}{argument\xspace}
\providecommand{\axioms}{arguments\xspace}

\providecommand{\phi}{\varphi}

\newcommand{\X}{\mathcal{A}}
\DeclareMathOperator{\ran}{ran}

\newcommand{\restrict}{\upharpoonright}

\begin{document}
%
% \title{Dialectical Systems and Belief Revision}
% \title{On the role of counterexample in AI}
% \title{The necessity of both contradiction and counterexample in iterated belief revision}
% \title{Truth by Trial: Dialectical Models of Adaptive Belief Revision}
% \title{From Contradictions to Counterexamples: Enhancing Belief Revision in Dialectical Systems}
% \title{Reasoning with Feedback: A Comparative Study of Dialectical Systems}
% \title{Learning from Mistakes: Dialectical Systems and the Logic of Belief Change}
% \title{Contradiction, Counterexample, and the Evolution of Belief}
% \title{On the Expressive Strength of Dialectical Systems in Belief Revision}
\title{Comparing Dialectical Systems:\\ Contradiction and Counterexample in Belief Change (Extended Version)}

\titlerunning{Comparing Dialectical Systems (Extended Version)}
% If the paper title is too long for the running head, you can set
% an abbreviated paper title here
%
\author{Uri Andrews\inst{1}\orcidID{0000-0002-4653-7458}
%\thanks{This research was supported by the NSF grant DMS-2348792.} 
\and
Luca San Mauro\inst{2}\orcidID{0000-0002-3156-6870}
%\thanks{San Mauro is a member of INDAM-GNSAGA.} 
}
\authorrunning{U.\ Andrews, L.\ San Mauro}
% First names are abbreviated in the running head.
% If there are more than two authors, 'et al.' is used.
%
\institute{University of Wiscosin--Madison, Madison, WI 53706, USA
\email{andrews@math.wisc.edu}\\
\url{http://math.wisc.edu/$\sim$andrews} \and
University of Bari, Bari, Italy
\email{luca.sanmauro@gmail.com}\\
\url{https://www.lucasanmauro.com/}}
\maketitle              % typeset the header of the contribution
\begin{abstract}

% Dialectical systems are a mathematical formalism  for modeling an agent engaged in iterated belief revision over a stream of arguments. Originally introduced to capture how  a working mathematician---or a research community---refines beliefs in the pursuit of truth, these systems also serve as natural models for the iterated belief revision of an automated agent.

Dialectical systems are a mathematical formalism for modeling an agent updating a knowledge base seeking consistency. Introduced in the 1970s by Roberto Magari, they were originally conceived to capture how a working mathematician or a research community refines beliefs in the pursuit of truth. Dialectical systems also serve as natural models for the belief change of an automated agent, offering a unifying, computable framework for dynamic belief management.

The literature distinguishes three main models of dialectical systems: (d-)dialectical systems based on revising beliefs when they are seen to be inconsistent,
p-dialectical systems based on revising beliefs based on finding a counterexample, and q-dialectical systems which can do both.
% In the original (d-)dialectical systems, belief revision occurs when  
%when the agent detects an inconsistency in its current set of belief. Revision is carried out by identifying the smallest subset of beliefs that is inconsistent (contraction) and removing an additional offending argument (excision). In p-dialectical systems, the agent reacts not to contradictions but to counterexamples. When a counterexample is encountered, the agent contracts to the minimal subset responsible for it, and then performs replacement, substituting a belief with another that is not currently refuted by any known counterexample. Finally, a q-dialectical system can find either inconsistency or counterexamples, and can perform the corresponding action in response to each.

We answer an open problem in the literature by proving that q-dialectical systems are strictly more powerful  than p-dialectical systems, which are themselves known to be strictly stronger than (d-)dialectical systems. This result highlights the complementary roles of counterexample and contradiction in automated belief revision, and thus also in the reasoning processes of mathematicians and research communities. 
%Also, in the context of automated belief revision,  it underscores the necessity of incorporating both forms of feedback for adaptive reasoning.

\end{abstract}

\section{Introduction}

The question of how a rational agent updates their beliefs in response to new information is fundamental to artificial intelligence research and central to the field of belief revision. Belief revision, 
 surveyed e.g.\ in \cite{ferme2018belief}, systematically investigates the logical and epistemological principles guiding changes in an agent's beliefs.

Modern belief revision is primarily characterized by the AGM framework, introduced by Alchourrón, Gärdenfors, and Makinson \cite{AGM}. Within AGM, an agent's beliefs are represented as sentences forming a deductively closed \emph{belief set}. \emph{Revision} becomes necessary when a new sentence contradicts the existing belief set. To preserve consistency, AGM incorporates an initial step known as \emph{contraction}. Central to this step is the concept of a \emph{remainder}, defined as the family of inclusion-maximal subsets of a given belief set that are consistent %\texttt{MAJOR CHANGE HERE -- FROM DON'T IMPLY PHI TO AREN'T INCONSISTENT WITH PHI} 
with the new sentence which the agent aims to incorporate. Further refinement is achieved through \emph{epistemic entrenchment}, a partial ordering reflecting the epistemic values the agent assigns to the elements of the belief set, thus determining which beliefs the agent is more willing to abandon, if needed.  AGM's principal innovation, \emph{partial meet contraction}, leverages epistemic entrenchment to select among maximal consistent subsets, ensuring minimal informational loss. 

Despite its theoretical elegance, the AGM model presupposes idealized agents endowed with unrealistic cognitive abilities. Specifically, for sufficiently expressive logical languages---including those susceptible to Gödelian incompleteness phenomena---determining whether a belief set is consistent with a given sentence is generally undecidable. Consequently, the computation of remainders becomes intractable, yet classical AGM-based revision procedures assume these operations to be executable in a single step. Recognizing these limitations has spurred interest in alternative frameworks better aligned with realistic computational constraints (see, e.g, \cite{dalal1988investigations,gardenfors1990dynamics,harman1986change,hansson1992defense,hansson1994kernel,wassermann1999resource,garapa2017advances}).

% In this paper, we explore a particular class of frameworks, known as \emph{dialectical systems} (along with their natural variants),  which are apt to model the \emph{internal} processes that an agent goes through to eventually arrive at a consistent belief state, rather than responding to \emph{external} revision inputs as in the AGM-based approaches. 

In this paper, we investigate a distinctive class of frameworks known as \emph{dialectical systems}---along with their natural variants---which are well-suited to model the \emph{internal} processes by which an agent arrives at a consistent belief state, in contrast to the \emph{external} revision operations characteristic of AGM-style approaches. These systems avoid reliance on unrealistic assumptions and provide a computational perspective on belief formation. Dialectical systems also offer algorithmic models for a range of problems in artificial intelligence, such as extracting a consistent subset from an inconsistent knowledge base, or capturing a single step of belief revision in the AGM tradition. These and related applications are discussed in detail in Section~\ref{sec:app total}.

%In this paper, we explore a particular class of frameworks, known as \emph{dialectical systems} (along with their natural variants), which model a computational approach, free from unrealistic assumptions, to restoring consistency within inconsistent knowledge bases. Dialectical systems can also be used as models of algorithmic solutions to several problems in artificial intelligence including an agent trying to form a consistent belief set from an inconsistent knowledge base, a single step of belief revision \'a la AGM, or iterated belief revision where an AGM-like framework is applied infinitely often to edit a knowledge base due to a possibly infinite stream of new knowledge. Each of these applications are discussed in Section \ref{sec:applications}

Dialectical systems, originating from the work of Roberto Magari and his school in the 1970s \cite{Magari,gnani1974insiemi,bernardi1974aspetti,montagna1996logic} (alongside the conceptually related Jeroslow's experimental logics \cite{Jeroslow,hajek1977experimental}), were initially proposed to model the evolution of real mathematical theories  and the process of discovery by a working mathematician or an entire mathematical community.  These systems explicitly employ trial-and-error methodologies, where \emph{\axioms} (or \emph{axioms}, following Magari's terminology)  are provisionally accepted and subsequently revised in response to emerging inconsistencies---events which are inherently unpredictable by computational means. Recently revived through the lens of computability theory \cite{dialectical1,dialectical2,dialectical3}, dialectical systems have re-emerged as fertile ground for contemporary investigation.

%In this paper, we explore a particular class of frameworks, known as \emph{dialectical systems} (along with their natural variants), as \emph{a dynamical, iterated belief revision approach free from noncomputable assumptions}, and thus prone to algorithmic implementation. Dialectical systems, originating from the work of Roberto Magari and his school in the 1970s \cite{Magari,gnani1974insiemi,bernardi1974aspetti,montagna1996logic} (alongside the conceptually related Jeroslow's experimental logics \cite{Jeroslow,hajek1977experimental}), were initially proposed to model the evolution of real mathematical theories  and the process of discovery by a working mathematician or an entire mathematical community.  These systems explicitly employ trial-and-error methodologies, where \emph{\axioms} (or \emph{axioms}, following Magari's terminology)  are provisionally accepted and subsequently revised in response to emerging inconsistencies---events which are inherently unpredictable by computational means. Recently revived through the lens of computability theory \cite{dialectical1,dialectical2,dialectical3}, dialectical systems have re-emerged as fertile ground for contemporary investigation.

In this paper, we address an open problem concerning the comparative expressive power of natural variants of dialectical systems, elucidating connections between two complimentary forms of performing belief change. 
%the concepts of \texttt{IS THE FOLLOWING THE RIGHT WORD? I THOUGHT THIS WAS FOR ANYTHING -- "BELIEF REVISION" BEING THE BROAD TOPIC} \emph{revision} and \emph{replacement} in belief change.
To contextualize this, we briefly describe dialectical systems, reserving formal definitions for Section \ref{sec:background}.

In essence, a dialectical system operates via a consequence operator which acts, stage by stage, upon a collection of provisionally accepted \axioms. At each stage, such a collection of \axioms expands by incorporating new \axioms drawn from a computable stream. Upon encountering \emph{contradictions}, a contraction procedure, guided by a naturally defined epistemic entrenchment ordering, restores consistency by selectively removing problematic \axioms. The expressive strength of the system is determined  by %the consequence closure of 
the set of \axioms eventually accepted,  which constitutes the system's limiting belief set.

The p-dialectical and the q-dialectical systems, introduced and thoroughly investigated  in \cite{dialectical1,dialectical2,dialectical3}, enhance basic dialectical systems by incorporating an \axiom replacement mechanism triggered by the emergence of \emph{counterexamples}. Unlike contradictions, which lead solely to the removal of problematic \axioms, counterexamples make it possible to retain part of the informational content of a discarded \axiom by \emph{replacing} it with an alternative. For instance, if the \axiom $a$ asserts that ``all prime numbers are odd'', the counterexample triggered by the prime number $2$ prompts a replacement such as ``all prime numbers greater than 2 are odd''.

In p-dialectical systems, the consequence operator produces only counterexamples, so replacement is always required when contraction occurs.
In contrast, q-dialectical systems may produce both counterexamples and contradictions, with replacement occurring in the former case and simple removal in the latter.

%In p-dialectical systems, the consequence operator exclusively yields counterexamples, mandating replacement upon contraction; conversely, q-dialectical systems may produce both counterexamples and contradictions, allowing---but not requiring---replacement. \texttt{THIS SEEMS A LITTLE MISLEADING TO ME -- IT'S NOT LIKE THE SYSTEM GETS A CHOICE}

Prior work \cite{dialectical2} established that both p- and q-dialectical systems surpass basic dialectical systems in expressive power but left open (see \cite[Problem 2.3]{dialectical3}) the relative expressiveness between the two variants. In Section \ref{sec:main results}, we solve this problem by showing that q-dialectical systems possess strictly greater expressive power than p-dialectical systems. This suggests that the ability to perform contractions both with and without replacement enhances the capabilities of rational agents. Our proof employs techniques from computability theory, specifically the finite injury priority method. Finally, Section \ref{sec:discussion} elaborates on implications of these findings for broader discussions in the belief revision literature.

\section{Background}\label{sec:background}

For any set $X$, we denote by $P(X)$  the \emph{power set} of $X$: i.e., the collection of all subsets of $X$. By $\pfin(X)$, we denote the collection of finite subsets of $X$. A function $f$ is acyclic if $f^n(x)\neq x$ for every $n\geq 1$ and $x$. Here $f^n(x)$ represents the $n$th iterate of $f$ applied to $x$.

A \emph{string} is a finite sequence of elements from a given set. For a string $\sigma$, $|\sigma|$ denotes its length;  $\sigma(n)$ denotes its $n$th element; and $\sigma\restrict_k$ denotes the string formed by taking the $k$ initial elements of $\sigma$.
 If a string $\tau$ coincides with $\sigma\restrict_k$, for some $k$, we write $\tau\preceq\sigma$ and say that $\tau$ is a \emph{prefix} of $\sigma$. Finally, the string $\sigma^\smallfrown \tau$ denotes the string formed by concatenating $\sigma$ and $\tau$.
 
% \texttt{Acyclic function}

We often make reference to (partial) \emph{computable functions}. These are the (partial) functions computed by a Turing machine. We use a standard indexing where $\varphi_e$ is the partial computable function whose instruction-set is encoded by the number $e$. The \emph{computably enumerable} (c.e.\ sets) are the ranges of partial computable functions. %The Halting Problem is the decision problem of determining whether a given Turing machine halts on a given input---or, equivalently, whether an arbitrary computation converges. It serves as the paradigmatic example of an undecidable problem.
%FROM URI: I refuse to believe there is a CS person on earth who doesn't know what the halting problem is.

%\subsection{Computability background}
%
%computable functions, c.e.\ sets, canonical indices, pairing functions, enumeration operators (?) \ldots
%\texttt{What do we actually use? Certainly computable functions, but maybe that's all.}

\subsection{Consequence and approximated consequence operators}

Our agents will revise their beliefs in response to seeing either a contradiction or counterexample arise from the set of \axioms that they currently accept. In aiming to model realistic reasoning, we do \emph{not} assume that it is immediately evident whether a given set of  \axioms is contradictory or admits a counterexample.  As shown by G\"odel's incompleteness theorems \cite{GodelI,Shoenfield1967}, even in highly idealized or extremely simple settings,
determining the inconsistency of a set of statements is as hard as solving the Halting Problem. It is therefore completely unrealistic to expect either an automated reasoner or a working mathematician  to reliably determine whether a belief set is consistent. Instead, we adopt the notion of an \ado, whereby a reasoner observes, over time, an increasing sequence of consequences derived from their \axioms. If the \axioms are inconsistent,  a contradiction will eventually be revealed at some finite stage (corresponding to a halting computation); however, no finite stage can ever certify consistency, as this would amount to verifying that a contradiction never arises (which corresponds to proving that a computation never halts).

Our consequence operators include two additional logical symbols: $\bot$, representing a contradiction, and $\ce$, representing  a counterexample. In many logical systems---arithmetic being a standard example---the statement $0 = 1$ is conventionally used to denote falsity, thereby playing the role of $\bot$. To maintain full generality in our framework, we introduce distinct symbols for contradiction and counterexample, without presupposing any specific properties of the set of \axioms $\A$.

\begin{definition}
%	For any set $\A$, $P(\A)$ is the power set of $\A$, and $\pfin(\A)$ is the collection of finite subsets of $\A$.
	Let $\A$ be a non-empty set.   A \emph{Tarskian consequence operator on} $\X$ is a function $H^\infty$ from $P(\X)$ to $P(\X\cup \{\bot,\ce\})$ satisfying the following conditions, for all $F, G\in P(\X)$,
    \begin{itemize}
        \item \emph{(Monotony)} $F\subseteq G$ implies $H^\infty(F)\subseteq H^\infty(G)$;
        \item \emph{(Inclusion)} $F\subseteq H^\infty(F)$;
        \item \emph{(Iteration)} $H^\infty(H^\infty(F)\cap \X)=H^\infty(F)$.
    \end{itemize}

    An \emph{\ado on} $\X$ is a computable function $H$ from $\nat\times \pfin(\X)$ to $\pfin(\X\cup \{\bot,\ce\})$ satisfying the following, for all $n\in \nat$ and $F, G\in \pfin(\X)$, 
    \begin{itemize}
        \item $F\subseteq G$ implies $H(n,F)\subseteq H(n,G)$;
        \item $F\subseteq H(n,F)\subseteq H(n+1,F)$;
%        \item For each $F\in \pfin(X)$, $F\subseteq H_0(F)$.
        \item letting $H^\infty(Y):=\bigcup_{n\in \nat, F\subseteq Y} H(n,F)$, we have that $H^\infty$ is a Tarskian consequence operator on $\X$.
    \end{itemize}
\end{definition}

\subsection{Representing current beliefs and operations of revision}

In a dialectical system, our agent forms beliefs about a countably infinite set of \axioms $\A$. At each stage $s$ the agent’s current belief state is represented by a string $\sigma_s$ of elements from 
$\A\cup \{\ast\}$. In each position, the string $\sigma_s$ either contains an \axiom from $\A$ or will contain the placeholder symbol $\ast$. The presence of $\ast$ at a given position indicates that the agent has encountered a reason to reject the corresponding \axiom, whereas the presence of the \axiom itself indicates continued belief in its truth.

%In a dialectical system, our agent  forms beliefs about a set of \axioms $\A$. At a given stage $s$, it will represent its current belief via a finite sequence of \axioms from $\A$ along with the symbol $\ast$. This represents the current belief in the truth of each \axiom which appears in $\A$. The symbol $\ast$ will be a place-holder to represent that the agent has already seen a reason to refute the \axiom that is missing in that place in the sequence. 

For a string $\sigma$ from $\A\cup \{\ast\}$, we let $\ran(\sigma)$ be the elements in $\A$ which appear in $\sigma$. %For the sake of exposition, we henceforth assume that all our sets of arguments are countably infinite.
%\footnote{This assumption does not restrict the generality of our analysis, as finite sets can be accommodated by allowing repetitions of arguments.}

\begin{definition} Let $\A:=\{a_i: i\in\nat\}$ be a set of arguments. We define four operations on strings to model how our agents revise their beliefs about the elements of $\A$:
    
    \begin{itemize}
        \item \emph{\textbf{Contraction:}} Given a string $\sigma$ and a number $k<\vert \sigma\vert$, this operation replaces $\sigma$  with its prefix $\sigma\restrict_k$;

  \item   \emph{\textbf{Expansion:}} Given a string 
 $\sigma$, this operation replaces $\sigma$ with $\sigma^\smallfrown a_{\vert\sigma\vert}$;

   \item  \emph{\textbf{Replacement:}} Given a string 
 $\sigma=\tau^\smallfrown a_i$, this operation  replaces $\sigma$ with $\tau^\smallfrown a_j$, for some $j\neq i$;

    \item \emph{\textbf{Excision:}} Given a string  $\sigma=\tau^\smallfrown a_i$, this operation replaces $\sigma$ by $\tau^\smallfrown \ast$.
    \end{itemize}
\end{definition}

These operations incorporate our two methods for revising the set of \axioms in response to either a contradiction or a counterexample. When a contradiction is detected, we apply \textbf{Contraction} followed by \textbf{Excision} to remove and mark the offending \axiom as contradictory. In the case of a counterexample, we instead apply \textbf{Contraction} followed by \textbf{Replacement}, allowing the original \axiom to be refined rather than discarded outright.

\subsection{Dialectical Systems}
We begin by formally defining a q-dialectical system. We note that we offer a re-design of the formal exposition of these systems, while ensuring that our definitions remain equivalent to the original ones; see 
 Appendix \ref{Appendix A} for a proof of such equivalence.
Dialectical systems and p-dialectical systems will then be introduced as special cases of q-dialectical systems\footnote{This ordering does not reflect the historical development of these notions---dialectical systems were introduced first, followed by q-dialectical and then p-dialectical systems---but we adopt this presentation for greater clarity.}. In a nutshell, a q-dialectical system is a framework for iterated belief revision  that accommodates both primary sources of revision: contradictions and counterexamples.

\begin{definition}\label{def:q-ours}
    A q-dialectical system $\Gamma$ is a triple $(\X, H, r)$ where $\X$ is some computable sequence $\{a_i : i\in \nat\}$ of \axioms, $H$ is an \ado on $\X$, and $r$ is a computable acyclic function from $\X$ to $\X$.    
\end{definition}

%\texttt{DO WE CARE THAT IT'S ACYCLIC? -- so we get loops? So what?}

The computable listing of the elements in $\A$ serves to impose an epistemic entrenchment ordering on the \axioms, which guides the choice of which remainder to retain (either via excision or replacement) during successive belief revision steps.

Restricting the system to respond to only one type of revision---either contradictions or counterexamples---yields the following  variants:

\begin{definition}
    A \emph{$p$-dialectical system} is a $q$-dialectical system in which contradiction are never produced: i.e., $\bot\notin H^\infty(\A)$. A \emph{dialectical system} (in the literature, also called a d-dialectical system) is a $q$-dialectical system in which counterexamples are never produced: i.e., $\ce\notin H^\infty(\A)$.
\end{definition}

The following definition specifies how our systems operate:

\begin{definition}\label{def:q-run-ours}
    Let $\Gamma$ be a $q$-dialectical system.  A \emph{run} of $\Gamma$ is an infinite sequence of strings $\sigma^\Gamma_i$ from $\A\cup \{\ast\}$, defined recursively as follows:

$\sigma^\Gamma_0$ is the empty string. 
Given $\sigma^\Gamma_s$, we define $\sigma^\Gamma_{s+1}$ by case distinction:
    \begin{itemize}
        \item \emph{Case 1:} $\bot$ or $\ce$ is in $H(s,\ran(\sigma^\Gamma_{s}))$. Let $k$ be least so that $\bot$ or $\ce$ is in $H(s,\ran(\sigma^\Gamma_{s}\restrict_k))$.
        \begin{itemize}
            \item If $\bot\in H(s,\ran(\sigma^\Gamma_{s}\restrict_k))$, then $\sigma^\Gamma_{s+1}=(\sigma^\Gamma_{s}\restrict_{k-1}) ^\smallfrown {\ast}$.
            \item Otherwise, $\sigma^\Gamma_{s+1}=(\sigma^\Gamma_{s}\restrict_{k-1})^\smallfrown r(\sigma^\us_s(k-1))$.
        \end{itemize}

        %     We apply contraction followed by excision. Let $k$ be greatest such that $\bot\notin H(s,\ran(\sigma^\Gamma_{s}\restrict_k))$, and define $\sigma^\Gamma_{s+1}=(\sigma^\Gamma_{s}\restrict_k) ^\smallfrown {\ast}$.
        % \item \emph{Case 2:} Not in the previous case, and $\ce\in H(s,\ran(\sigma^\Gamma_{s}))$.
        
        %     We apply  contraction followed by replacement. Let  $k$ be greatest such that $\ce \notin H(s,\ran(\sigma^\Gamma_{s}\restrict_k))$ and let $\sigma^\Gamma_{s+1}=(\sigma^\Gamma_{s}\restrict_k)^\smallfrown r(\sigma(k))$.
        \item \emph{Case 2:} Not case 1. We apply expansion. Let $\sigma^\Gamma_{s+1}=(\sigma^\Gamma_{s})^\smallfrown a_{\vert \sigma^\Gamma_s\vert}$.
    \end{itemize}

\end{definition}

\subsection{Limiting Belief Sets}

% \begin{definition}
%     A q-dialectical system is called loopless if, for each $n\in \nat$, there exists some $s\in \nat$ so that $\sigma_t(n)=\sigma_s(n)$ for every $t\geq s$. We write $\lim_s \sigma_s(n)$ for this value.
% \end{definition}

% We note that every dialectical system is loopless, but p- and q-dialectical systems may not be loopless, since the replacement function $r$ may iteratively suggest false revisions without us finding a contradiction. For example, 

\begin{definition}
    Let $\Gamma$ be a q-dialectical system, and let $n\in\nat$. If there exists some \axiom $a\in \A$ so that $\sigma^\Gamma_t(n)=a$ for all sufficiently large $t$, then we write $\lim_s \sigma^\Gamma_s(n)$ for $a$.
If $\lim_s \sigma^\Gamma_s(n)$ exists for every $n$, then we call $\Gamma$ \emph{loopless}. 
\end{definition}

Note that the determine if a given q-dialectical system is loopless is not a decidable property.

\begin{definition}
    For $\Gamma$ a q-dialectical system, the \emph{limiting belief set of $\Gamma$} is
\[    
  B_\Gamma = \{a\in \A : \exists n (a=\lim_s \sigma^\Gamma_s(n))\}=\ran(\lim_s\sigma_s^\us).
  \]
\end{definition}

Thus, the limiting belief set of a system $\Gamma$ comprises the arguments that are ultimately accepted (hence, believed) by the system.

\subsection{Related Work}
We briefly summarize the most pertinent results from the literature:
Firstly, the limiting belief sets are deductively closed.

\begin{theorem}[{\cite[Lemma 3.14]{dialectical1}}]
    If $\Gamma$ is a loopless q-dialectical system, then $B_\Gamma=H^\infty(B_\Gamma)$.
\end{theorem}

Next, dialectical systems, p-dialectical systems, and q-dialectical systems each have limiting belief sets in exactly the c.e.\ Turing degrees.

\begin{theorem}[{\cite[Theorem 3.2]{dialectical2}}]
    For any q-dialectical system $\Gamma$, the Turing degree of $B_\Gamma$ is the Turing degree of a c.e.\ set. Conversely, if $C$ is a c.e.\ set, then there is a dialectical system $\Gamma$, and a p-dialectical system $\Gamma'$ so that $B_\Gamma$ and $B_{\Gamma'}$ are in the same Turing degree as $C$.
\end{theorem}

Despite the complexity of their limiting belief sets being the same, p-dialectical systems are strictly stronger than dialectical systems.

\begin{theorem}[{\cite[Theorem 2.1 and Corollary 6.11]{dialectical3}}]
    If $\Gamma$ is a dialectical system, then $B_\Gamma$ is the limiting belief set of some p-dialectical system $\Gamma'$.    
    
       However, there are p-dialectical systems $\Gamma$ so that $B_\Gamma$ is not the limiting belief set of any dialectical system.
\end{theorem}

% These results demonstrate that systems equipped with the operation of replacement (p-dialectical systems) are strictly more expressive than those limited to excision (dialectical systems). 
In \cite{dialectical3}, the following question was posed regarding the relative strength of q-dialectical and p-dialectical systems:

%This shows that systems equipped with the operation of replacement (p-dialectical systems) are strictly stronger than systems equipped with the operation of excision (dialectical systems).  In \cite{dialectical3}, the following question was posed regarding the relative strength of q-dialectical and p-dialectical systems:

\begin{question}[{\cite[Problem 2.3]{dialectical3}}]
   Is there a q-dialectical system $\Gamma$ such that $B_\Gamma$ is not the limiting belief set of any p-dialectical system?
    \end{question}
    
    In other words, does an iterated belief revision agent \emph{require} the notion of contradiction (i.e., the operation of excision), or is the concept of counterexample (i.e., the operation of replacement) alone sufficient?
    
    \smallskip

In the next section, we resolve this question by proving that q-dialectical systems are strictly more expressive than p-dialectical systems.

\section{Applications of Dialectical Systems}\label{sec:app total}

We describe natural applications of dialectical systems to problems arising in artificial intelligence, particularly in the management and evolution of knowledge under conditions of inconsistency and change.

\subsection{Repairing an Inconsistent Knowledge Base}\label{sec:app 1}

A fundamental application is to the case of an agent equipped with a knowledge base $K$ that may be internally inconsistent. Suppose further that $K$ is equipped with an entrenchment ordering reflecting the relative importance or reliability of its elements. As the agent draws inferences, contradictions or counterexamples may arise. In response, the agent applies excision or revision operations as specified by a corresponding (q-)dialectical system $\us$. In many classical models, inconsistency is handled purely syntactically and does not accommodate the notion of counterexamples. In such cases, $\us$ may be taken as a d-dialectical system. Thus, the dialectical framework provides a theoretical model for an algorithmic procedure to manage and repair inconsistent knowledge bases.

\subsection{Belief Revision}\label{sec: app 2}

In belief revision, a consistent knowledge base $K$ must be updated to incorporate a new belief $b$ given a Tarskian consequence operator $H^\infty$ and an entrenchment order on $K$, yielding a revised set $K'$ that includes $b$ and retains as much of $K$ as possible while preserving consistency. This is the central concern of the AGM framework. We model this phenomenon within the dialectical model, by defining a dialectical system $\us$ whose collection of arguments is $K$ with the given entrenchment ordering. We then define $H$ by setting $H(n,F)=H'(n,F\cup \{b\})\cap (K\cup \{\bot\})$ where $H'$ is an \aco limiting to $H^\infty$. The limiting belief set of $\us$ is a set $K'$ so that $K'$ preserves a subset of $K$ which is consistent with $b$. 
%Further, taking $<$ to be the lexicographic ordering of subsets of $K$ induced by the  entrenchment ordering on $K$, $K'$ is the $<$-maximal subset of $K$ which is consistent with $b$. 

In this setting, we consider a single new argument $b$ which we accept as externally given truth and we revise $K$ accordingly. Essentially the same method allows us to accept a computable sequence of externally given truths $(b_i)$. We simply define $H(n,F)=H'(n,F\cup \{b_i\mid i\leq n\})\cap (K\cup \{\bot\})$.

% To revise $K$ with $b$, we define a new dialectical system $\us'$ by inserting $r$ into the initial set $H(0, \emptyset)$ of the associated \aco. The resulting limiting belief set $K'$ of $\us'$ includes $b$ and a maximal consistent fragment of $K$ that is compatible with it. This provides a dynamic and algorithmic approach to belief revision consistent with dialectical reasoning.

% \subsection{Sequential Belief Revision}

% When presented with a computable sequence of new beliefs $(b_i)$, we consider the task of updating an initial consistent knowledge base $K$ to reflect this stream of information and maintain consistency with a given Tarskian consequence operator $H^\infty$. Standard methods perform iterated belief revision: incorporating $b_0$ into $K$, then updating the result with $b_1$, and so on. However, due to the undecidability of consistency, each of the sequential revision processes requires an infinite number of steps, and thus the process requires a transfinite process to limit towards a limiting belief set.

% In contrast, the dialectical framework accommodates iterated belief revision more naturally. As in \ref{sec: app 2}, we define an \aco $H$ so $H(i,F)=H'(i,F\cup \{b_j \mid j<i\})\cap (K\cup \{\bot\})$ where $H'$ is an \aco limiting to $H^\infty$, we define a dialectical system $\us'$ whose limiting belief set $K'$ is a fragment of $K$ compatible with the sequence $\{b_i \mid i \in \mathbb{N}\}$. This approach enables concurrent revision steps based on both the incoming stream of new beliefs and the evolving set of derived consequences.

\medskip

In each of these applications, dialectical systems offer \emph{a unifying, computable framework for dynamic belief management}. They support the reconciliation of inconsistency, enable rational belief revision, and handle iterative updates in a principled and theoretically grounded manner.

\section{The Main Theorem}\label{sec:main results}

We construct a q-dialectical system $\us$ with the aim of ensuring, via diagonalization, 
that $B_\us \neq B_{\them}$ for every p-dialectical system $\them$. To achieve this, we require a computable enumeration of all p-dialectical systems:

\begin{definition}
	We fix a computable collection $\A=\{a_n: n\in \nat\}$ of \axioms, and a computable bijection $\pi$ between the natural numbers and the elements of $\pfin(\A\cup \{\ce\})$.

Next,	for $i_0,i_1,i_2\in \nat$, we define a \emph{partial p-dialectical system} $(\A',H,r)$ for $(i_0,i_1,i_2)$ as follows:
	\begin{itemize}
		\item $\A'=\{g_n \mid n\in \nat\}$, where $g_n=a_{\varphi_{i_0}(n)}$;
		\item $H(s,X)=\bigcup_{t\leq s} \pi({\varphi_{i_1}(t,\pi^{-1}(X))})$;
		\item $r(g_n)=g_{\varphi_{i_2}(n)}$.
	\end{itemize}

	For any $m\in \nat$ where $m=2^{i_0}3^{i_1}5^{i_2}\cdot k$ where $2,3,5$ do not divide $k$, we let $\them_m$ be the above partial p-dialectical system for $(i_0,i_1,i_2)$.
\end{definition}

Note that for $j \leq 2$, the functions $\varphi_{i_j}$ may be partial, and the resulting operator $H^\infty$ may fail to satisfy the axioms of a Tarskian consequence operator. Consequently, $\them_m$ does not represent a valid p-dialectical system for every $m$. However, every p-dialectical system is represented by some $\them_m$ for an appropriate $m \in \mathbb{N}$.

\begin{theorem}\label{thm:Main}
    There is a loopless q-dialectical system $\us$ such that $B_\us$ is not the limiting belief set of any p-dialectical system.
\end{theorem}
\begin{proof}
    We fix $\A=\{a_i \mid i\in \nat\}$ and aim to define a q-dialectical system $\us = (\A, H, r)$ by specifying computable functions $H$ and $r$. In the ``construction'' phase of the proof, we explicitly describe the algorithms for computing these functions. At each stage $s$, an instruction to ``put $a$ into $H(s,F)$'' means that we set $H(s,F):=H(s-1,F)\cup \{a\}$.

 The core idea of the proof is to diagonalize against the possibility that $B_\us=B_{\them_i}$ for any partial p-dialectical system $\them_i$.  To this end, we introduce an infinite collection of strategies $\Strat_i$, one for each $i\in \nat$. Each strategy $\Strat_i$ will act as a module executed by the algorithm constructing $\us$, with the specific goal of ensuring that $B_\us \neq B_{\them_i}$.
 
To clarify the overall construction, we begin by informally describing how we might construct a q-dialectical system to satisfy a single strategy, $\Strat_0$, whose task is to guarantee that $B_\us \neq B_{\them_0}$. For simplicity, in the next subsection, we let $\them := \them_0$. 
        
 %   We will first informally describe a single strategy $\strat_0$ to ensure $B_\us\neq B_{\them_0}$. For this description, we focus on only one $\them$, so we let $\them=\them_0$.

    \subsection*{An informal description of one strategy to ensure $B_\us\neq B_\them$.}
    This strategy consists of two parts. \textbf{Part 1} identifies pair of \axioms $a_i,a_j$ and stage $s$ so that $a_i,a_j$ appear in opposite orders in $\sigma_s^\us$ and $\sigma_s^\them$.
	\textbf{Part 2} then exploits this mismatch to ensure that $B_{\us}\cap\{a_i,a_j\}\neq  B_{\them}\cap\{a_i,a_j\}$.

 \subsubsection*{Part 1:} We fix the \axioms $a_0, a_1, a_2$ to be used by this strategy and define $r(a_0)=a_2$. Then we wait  for a stage at which we find $l<m<n$ such that $\{g_l,g_m,g_n\}=\{a_0,a_1,a_2\}$ and  $\sigma_s^\them$ has length greater than $>n$. 
    %and so that $\ran(\sigma^\them_s\restrict_n+1)\subseteq \ran(\sigma_s^\us)$. 
  We also wait until $r^\them(g_l)$ is obtained.  
  Note that we are currently building $B_\us=\A$: i.e.,  our \aco does not produce any contradictions or counterexamples. If, during this waiting period, $\them$ produces  any counterexample, then some \axiom will leave $B_\them$ while remaining in $B_\us$, and this will give us  a trivial victory for ensuring $B_\them \neq B_\us$.

Assuming such a stage is reached, we now distinguish cases. If the order of $a_0,a_1,a_2$ in $\sigma^\them_s$ is different than the order in $\sigma^\us_s$, then we are already done with Part 1. Thus, we may assume $g_l=a_0$, $g_m=a_1$, and $g_n=a_2$. Our action now depends on the value of $r^\them(g_l)$. 

    \begin{itemize}
        \item \emph{Case 1: $r^\them(g_l)=a_1$}\\ We put $\ce$ into $H(s,\{a_0\})$. This triggers replacement of $a_0$ with $a_2$ in $\sigma^\us_{s+1}$, placing $a_2$ before $a_1$. In $\sigma^\them_t$ for $t>s$, to let $a_0\notin B_\them$, the system must replace $a_0$ with $a_1$, thus putting $a_1$ before $a_2$. Hence, the ordering of $a_1$ and $a_2$ will differ between $\sigma^\them_t$ and $\sigma^\us_t$ at some stage $t>s$.

        \item \emph{Case 2: $r^\them(g_l)=a_2$}\\ We put $\bot$ into $H(s,\{a_0\})$. This excises $a_0$ in $\sigma^\us_{s+1}$, leaving $a_1$ before $a_2$ in $\sigma^\us_t$ for each $t>s$. In $\sigma^\them_t$ for $t>s$, to let $a_0\notin B_\them$, the system must  replace it with $a_2$, yielding $a_2$ before $a_1$---again, a mismatch.

        \item \emph{Case 3: $r^\them(g_l)=g_i$, for some $g_i\notin \{a_1,a_2\}$}\\
        We ensure that $g_i$ remains in $B_\them$, so that this \axiom will not be replaced in $\sigma^\them_t$ for any $t>s$. Then, we put $\ce$ into $H(s,\{a_0\})$, replacing $a_0$ with $a_2$ in $\sigma^\us_{s+1}$. This puts $a_2$ before $a_1$ in $\sigma^\us_t$ for each $t>s$, while in $\sigma^\them_t$, $a_1$ will precede $a_2$---again, a mismatch.
    \end{itemize}

    At this point, we have transitioned from building  $B_{\us}=\A$ to $B_\us=\A\smallsetminus \{a_0\}$. Thus, for $B_\them$ to equal $B_\us$, $\them$ can only derive $\ce$ in a way that removes $a_0$.  Any additional replacement will necessarily cause $B_\them \neq B_\us$. 
     In this way, we ensure that once $\sigma^\them_t$ grows large enough, we will see the intended order difference. (Note: We will have to be more careful when we discuss the full construction with many simultaneously running modules, each working towards different strategies). 

\subsubsection*{Part 2:}  Now that we have identified a pair of \axioms $a_i$ and $a_j$ that appear in opposite orders in $\sigma^\them_t$ and $\sigma^\us_t$, we force a difference in the corresponding limiting belief sets.

Suppose that $a_i$ appears first in $\sigma^\them_t$ and $a_j$ appears first in $\sigma^\us_t$. We add $\bot$ to $H(t+1,\{a_i,a_j\})$. This causes $a_i$ to leave the range of $\sigma^\us_{t+1}$.

If $B_\them = B_\us$ were to hold, then at some later stage, $\them$ must derive $\ce$ on $a_i$, i.e., insert $\ce$ into $H(t', \sigma^\them_t \restrict_{(k+1)})$ where $\sigma^\them_t(k) = a_i$.  If this happens, we will finally put $\bot$ into $H(x+1,\{a_j\})$ ensuring that $a_j$ will not be in $B_{\us}$. After this action,  $a_i$ re-enters $\sigma^\us_y$, since its only conflict was with $a_j$, and $a_j$ has been removed. However, $a_i$ cannot re-enter $\sigma_y^{\them}$ without $\them$ removing some other \axiom that still belongs to $B_\us$, again guaranteeing $B_\us \neq B_\them$.

\medskip

We now discuss how to fit together multiple strategies, and then give a more formal algorithm for each strategy.

\subsection*{Fitting multiple strategies together.}

A central feature of our strategy for ensuring $B_\us \neq B_\them$ is maintaining $B_\us$ sufficiently large so that few \axioms in $\sigma^\them_s$ are candidates for replacement. In particular, we must prevent any element $\sigma^\them_s(j)$ for $j \in [0, m]$ from being replaced by $a_2$, which would otherwise reverse the intended order of $a_1$ and $a_2$ in $\sigma^\them_t$. 
We accomplish this by ensuring that all \axioms in $\ran(\sigma^\them_s \restrict_{(n+1)}) \smallsetminus \{a_0\}$ remain in $B_\us$.

When multiple strategies $\strat_i$ run in parallel, no single strategy can simply assume that $B_\us = \A$ unless it chooses to add a contradiction or counterexample to $H$.
% nor even that $B_\us = \A \setminus {a_0}$, since $\strat_0$ may remove more than one \axiom. 
We can also not wait for $\strat_0$ to finish before beginning $\strat_1$, since $\strat_0$ might remain indefinitely waiting for $\them_0$ to exhibit some  specific behavior. Despite this, other strategies $\strat_j$ for $j > 0$ must still be able to proceed and succeed. That means they must be allowed to introduce their own consequences into $H$---involving contradictions or counterexamples---without interfering with the goals of $\strat_0$. Ensuring this kind of mutual coherence between strategies is essential to the success of the overall construction.

To manage these interactions, we employ a classical method from computability theory: the \emph{finite-injury priority method}, introduced independently by Friedberg \cite{friedberg1957two} and Muchnik \cite{muchnik1956negative}. The core idea is to resolve conflicts between strategies by prioritizing the needs of $\strat_i$ over the needs of $\strat_j$ when $i<j$. This ensures that each strategy $\strat_i$, being only \emph{injured} (i.e., reset) by the finitely many higher-priority strategies $\strat_j$ for $j<i$, eventually will be allowed to stabilize and succeed.

In a typical finite-injury construction, whenever a strategy $\strat_i$ acts, all lower-priority strategies, i.e., $\strat_j$ with $j > i$, are deactivated. These lower-priority strategies must then restart from scratch. Since each $\strat_i$ acts only finitely often, every lower-priority strategy will eventually face no further interference and will complete its task.

A fundamental principle of the priority method is that lower-priority strategies must never cause any harm to higher-priority ones. Accordingly, when defining our strategies in the next subsection, we will carefully design each one so that its actions are harmless to those of higher priority.

Concretely, suppose strategy $\strat_i$ intends to put, say, $\ce$ into $H(t, F)$, for some finite set of arguments $F$. To avoid harming higher-priority strategies, it will ensure that $F$ contains any set $X$ that a higher-priority strategy might later use to derive a contradiction or a counterexample. This guarantees that, should the higher-priority strategy act, the counterexample based on $F$ becomes moot, and no conflict arises.

We now explicitly describe the algorithm performed by each strategy $\strat_i$.

    \subsubsection*{Part 1 of $\strat_i$:}\hfill
    
    \smallskip
    
    \textbf{Step 1}: We fix three fresh \axioms $a_N,a_{N+1},a_{N+2}$ such that $N$ is larger than any number ever mentioned in the construction before. We define the parameter $Z_i=\emptyset$.
    Let $S$ be the set of all arguments $a_k$ with $k<N$ which are not currently in  $\bigcup_{j<i} Z_j$.

 \emph{Comment}: If a higher-priority strategy (for $\them_k$ with $k<i$) has found his $g_l,g_m,g_n$ and considered $\sigma^{\them_k}_s$, then $N$ is larger than any $m$ so that $a_m\in \ran(\sigma^{\them_k}_s)$. By choosing this $N$ large, the strategy for $\them_i$ is making an effort to ensure that its actions will not affect any higher-priority strategy. As $N$ is  larger than the index of any mentioned \axiom, we have that for every finite set $F$, $H(s,F)=H(s,F\cup \{a_N,a_{N+1},a_{N+2}\})$. 
The parameter $Z_i$ keeps track of which elements this strategy  wants to remove from $B_\us$. Thus, the set $S$ includes all \axioms which a higher-priority strategy may, in the future, remove from $\ran(\sigma^\us_s)$. As described above, $S$ will be included in all of our derivations of contradictions or counterexamples.

\smallskip

    \textbf{Step 2}: \texttt{Wait} to see a stage $s$ at which we find $l<m<n$ with $\{g_l,g_m,g_n\}=\{a_N,a_{N+1},a_{N+2}\}$ and for $\sigma^{\them_i}_s$ to have length $>n$. Also, wait for $r^{\them_i}(g_l)$ to be defined.

\smallskip

    \textbf{Step 3}: We now act based on cases.
  
  \begin{itemize}
  \item 
     If $g_l \neq a_N$, then we have already ensured success in \textbf{Part 1}. Go to \textbf{Part 2} and deactivate all lower-priority requirements;
  \item If instead $g_l = a_N$, call the module \texttt{PredictOrder}$(\them_i, s, a_N, a_{N+1}, a_{N+2})$.
\end{itemize}

 \emph{Comment}: The goal here is to apply either excision or replacement to $a_N$ in order to force a specific order between $a_{N+1}$ and $a_{N+2}$ in $\sigma_t^\us$ that diverges from the order in $\sigma_t^{\them_i}$. However, a subtlety arises. Even if $a_{N+1}$ appears before $a_{N+2}$ in $\sigma_s^{\them_i}$, this order may change in future stages due to interactions with other \axioms.  For example, suppose an \axiom $c$, positioned between $a_N$ and $a_{N+2}$ in $\sigma_s^{\them_i}$, was replaced by $a_{N+1}$ as a result of a counterexample derived from the set $\{a_N, c\}$. If $a_N$ is later removed (through replacement or excision), then $c$ may reappear, and the original derivation of $\ce$ from ${a_N, c}$ will no longer apply. As a result, $a_{N+1}$ may not re-enter the string in place of $c$. Consequently, the eventual ordering of $a_{N+1}$ and $a_{N+2}$ in some future stage $\sigma_t^{\them_i}$ may differ from their ordering in $\sigma_s^{\them_i}$. 
  To address this, the module \texttt{PredictOrder}$(\them_i, s, a_N, a_{N+1}, a_{N+2})$, described below, returns a string $\rho$ representing a sufficiently long initial segment of $\sigma_t^{\them_i}$, computed at a suitably late stage $t$. The purpose of $\rho$ is to determine which of $a_{N+1}$ or $a_{N+2}$  appears first at this later stage: that is, if the last entry of $\rho$ is $a_K$ ($K\in \{N+1,N+2\}$), we will know that $a_{K}$ will come first in $\sigma^{\them_i}_t$.% at $t$.

 \smallskip

    \textbf{Step 4}: We perform the following action depending on the last entry of $\rho:=$ \texttt{PredictOrder}$(\them_i,s,a_N,a_{N+1},a_{N+2})$:

    \begin{itemize}
        \item \emph{Case 1:  $\rho$ ends in $a_{N+1}$}\\ We put $\ce$ into $H(s,S\cup \{a_N\})$ and let $Z_i=\{a_N\}$. 
        
%        Comment: With this, the run will perform replacement, replacing $a_N$ with $a_{N+2}$, thus putting $a_{N+2}$ before $a_{N+1}$ in $\sigma^\us_t$, but we will argue in Lemma \ref{lem:PredictOrder is right} below that either the strategy is successful or $a_{N+1}$ will appear before $a_{N+2}$ in $\sigma^\them_t$ for some stage $t$.

        \item \emph{Case 2: $\rho$ ends in $a_{N+2}$}.\\ We put $\bot$ into $H(s,S\cup \{a_N\})$ and let $Z_i=\{a_N\}$.

        % \item Case 3 $r^\them(g_l)=g_i$ for some $g_i\notin \{a_{N+1},a_{N+2}\}$. We have now officially mentioned $g_i$.  Then we put $\ce$ into $H(s,S\cup \{a_N\})$. This will perform replacement, replacing $a_N$ with $a_{N+2}$, thus putting $a_{N+2}$ before $a_{N+1}$ in $\sigma^\us_t$. We will argue in Lemma \ref{lem:Don't let shit roll uphill} that this will be enough to ensure that unless a higher-priority strategy acts, $g_i$ will appear in $B_\us$, and in Lemma \ref{} that either the requirement is satisfied or there will be a stage $t$ at which will have $a_{N+1}$ will appear before $a_{N+2}$ in $\sigma^\them_t$.
    \end{itemize}

\emph{Comment}:  These actions enforce a specific ordering in $\sigma_t^\us$: either replacement (putting $a_{N+2}$ before $a_{N+1}$) or excision (leaving $a_{N+1}$ before $a_{N+2}$). In Lemma~\ref{lem:PredictOrder is right}, we will prove that either the strategy succeeds or the ordering predicted by $\rho$ will persist in $\sigma_t^{\them_i}$ at some future stage $t$.

After performing the appropriate action, deactivate all lower-priority strategies and proceed to \textbf{Part 2}.

    \subsubsection*{Part 2 of $\strat_i$:}\hfill
    
    \smallskip

    \textbf{Step 5}: \texttt{Wait} for a stage $t$ so that $\rho\preceq \sigma^{\them_i}_t$. 
    
    \emph{Comment}: We will show in Lemma \ref{lem:PredictOrder is right} below that if no such stage $t$ exists, then $\strat_i$ still succeeds. This may occur because, e.g., $\them_i$ removes too many \axioms to match $B_\us$, or because $\them_i$ is not a valid p-dialectical system.
  
  \smallskip

    \textbf{Step 6}: Let $\{a_I,a_J\}=\{a_{N+1},a_{N+2}\}$ be so that the last entry of $\rho$ is $a_I$. We add $\bot$ to $H(t,S\cup \{a_I,a_J\})$ and update $Z_i:=\{a_N,a_I\}$.
    
    \emph{Comment}: This action triggers excision in the run of $\us$, removing $a_I$ from $\sigma_t^\us$. All lower-priority strategies are deactivated. As we will show in Lemma~\ref{lem:hands off}, this ensures that (assuming no further interference from an higher-priority strategy) the limiting belief set satisfies $\{a_I, a_J\} \cap B_\us = \{a_J\}$. Thus, either $B_\us \neq B_{\them_i}$ holds already, or a future stage $t'$ will witness the derivation $\ce \in H^{\them_i}(t', \rho)$.
    
  \smallskip
    
    \textbf{Step 7}: \texttt{Wait} for a $t'$ where $\ce\in H^{\them_i}(t',\rho)$.

\smallskip

    \textbf{Step 8}: Put $\bot$ into $H(x,S\cup \{a_J\})$, where $x$ is the current stage, and update $Z_i:=\{a_N,a_J\}$. Deactivate all lower-priority requirements, and declare $\strat_i$  complete.

    \emph{Comment}: This final step ensures  that $\{a_I, a_J\} \cap B_\us = \{a_I\}$. At the same time, the fact that $\ce\in H^{\them_i}(t',\rho)$ ensured that $a_I\notin B_\them$. Hence, $B_\us \neq B_{\them_i}$, completing the diagonalization.

    % At this point, we declare the strategy completed. We argue in Lemma \ref{lem:full strategies succeed} that unless a higher-priority strategy re-starts this strategy, being completed will ensure that $\{a_i,a_j\}\cap B_\us=\{a_i\}$, and that $B_\us\neq B_\them$.
    
\subsubsection*{The \texttt{PredictOrder} Module:}

\textbf{Step PO1:} Define $E$ to be $\{a_N\}\cup \bigcup_{j<i} Z_j$. 

\smallskip

\textbf{Step PO2:} WAIT for a stage $t$ where we see, for each $j<n$, enough convergences of $r^{\them_i}$ iterated on $g_j$ enough times that $(r^{\them_i})^e(g_j)\notin E$.

\emph{Comment}: Recall that if $\them_i$ is in fact a p-dialectical system, then $r^{\them_i}$ is an acyclic recursive function, so at some stage we should see each of these required convergences. 

\smallskip

\textbf{Step PO3}: Let $\tau$ be the string so $\tau(j)=(r^{\them_i})^{e_j}(g_j)$ where $e_j$ is least so that $(r^{\them_i})^{e_j}(g_j)\notin E$. 
Return the smallest prefix $\rho$ of $\tau$ containing either $a_{N+1}$ or $a_{N+2}$. 

\emph{Comment}: Note that $a_{N+1},a_{N+2}\notin E$, so $\tau(m)=g_m\in \{a_{N+1},a_{N+2}\}$, thus $\rho$ is well-defined as has length no more than $m$.

% Return ``$a_{N+1}<a_{N+2}$'' if $a_{N+1}$ appears before $a_{N+2}$ in $\tau$, ``$a_{N+2}<a_{N+1}$'' if $a_{N+2}$ appears before $a_{N+1}$ in $\tau$. Note that as a consequence of performing this computation, we have mentioned every $\tau(j)$ and each of the numbers $(r^{\them_i})^{\circ o'}(g_j)$ for $j<n$ and $o'\leq o_j$.

\smallskip

\noindent\textbf{Overarching program running each strategy:}
At stage 0 of the construction, all strategies are deactivated. 

At any given stage $s>0$, there may be some activated strategies and some deactivated strategies. If any of the activated strategies want to act (i.e., some \texttt{Wait} condition has been satisfied), we let the highest-priority strategy act according to the description of the strategies above. This deactivates all lower-priority strategies.

If none of the active strategies want to act, then we let $\them_k$ be the highest-priority strategy which is not yet activated, and we activate it. This means that it performs its first step of choosing new numbers $a_{N},a_{N+1},a_{N+2}$ and sets its \texttt{Wait} condition determining if it will later want to act again.

We also take the least $k$ so that $r(a_k)$ is not yet defined and let it equal $a_{k+1}$.

Comment: We note that we only ever put $\ce$ into $H(t,F)$ for any $t\in \nat$ and $F\in \pfin(\A)$ in Step 4 of some $\strat_j$. In this case, we use $r(a_N)$, which is defined in Step 1 of the same $\strat_j$ (recall $N$ was new, so $r(a_N)$ was not already defined when $a_N$ was chosen). Thus, we include this definition for $r$ in the construction in order that $r$ be a total function, but we will never perform replacement based on any of these values of $r$.

This completes the description of the construction of the q-dialectical system $\us$. We now shift to verifying the result that $B_\us\neq B_{\them_i}$ for every $i$.

\smallskip

\noindent\textbf{Verification:}
We proceed with a series of Lemmas which ensure that our construction does in fact produce a q-dialectical system $\us$ so that $B_\us\neq B_{\them_i}$ for each p-dialectical system $\them_i$.

Note that a strategy that is never deactivated after stage $s$ must have one of five possible outcomes: It could forever wait in Steps 2, 5, 7, PO2, or it may successfully get to Step 8 and complete the strategy. The following Lemmas, whose proofs are in Appendix \ref{sec:appx-verification} due to space constraints, follow via a careful analysis of the outcomes in each of these cases, along with a careful induction on the parameters of the strategies.

\begin{restatable}{lemma}{lemmahandsoff}\label{lem:hands off}
		    Suppose a strategy $\strat_i$ is activated at stage $s$ and is never deactivated after stage $s$. Let $N$ be the parameter chosen by $\strat_i$. Then $B_\us\cap \{a_i \mid i<N\} = \{a_i \mid i<N\}\smallsetminus \bigcup_{j<i} Z_j$. (Note that the value of $Z_j$ cannot change after stage $s$ since the higher-priority strategies do not act after stage $s$.)
\end{restatable}

%\vspace{-.1in}

\begin{restatable}{lemma}{lemmaEsAreRight}\label{lem: Es are right}
	Suppose a module \texttt{PredictOrder}$(\them_i,s,a_N,a_{N+1},a_{N+2})$ returns a value $\rho$ at stage $s$. Let $E$ be the set $E$ in the computation of \\ \texttt{PredictOrder}$(\them_i,s,a_N,a_{N+1},a_{N+2})$. Let $M$ be the largest number mentioned in the construction by stage $s$. Then either a higher-priority strategy than $\strat_i$ acts or $B_\us\cap( \{a_j \mid j\leq M\}\smallsetminus \{a_{N+1},a_{N+2}\}) = \{a_j \mid j\leq M\}\smallsetminus (\{a_{N+1},a_{N+2}\}\cup E)$.
\end{restatable}
%\vspace{-.1in}

\begin{restatable}{lemma}{lemmaPredictOrderIsRight}\label{lem:PredictOrder is right}
	Suppose that a strategy $\strat_i$ is not deactivated after stage $s$, and that \texttt{PredictOrder}$(\them_i,s,a_N,a_{N+1},a_{N+2})$ returns a value $\rho$. Then  either $B_\us\neq B_{\them_i}$ or at some stage $t$, we have $\rho\preceq \sigma^{\them_i}_t$. 
	
	Also, either $B_\us\neq B_{\them_i}$ or at all large enough stages, we have $\rho\restriction_{\vert \rho\vert -1} \preceq \sigma^{\them_i}_t$.
	
\end{restatable}

The following is standard in the finite-injury priority method.
\begin{restatable}{lemma}{lemmaStopInjury}\label{lem:stopInjury}
	Each strategy is deactivated at only finitely many stages.
\end{restatable}
%\vspace{-.1in}
\begin{restatable}{lemma}{lemmaWeWin}\label{lem:weWin}
	For each $i$, $B_\us\neq B_{\them_i}$.
\end{restatable}
This completes the proof of Theorem \ref{thm:Main}
\end{proof}

\section{Discussion}\label{sec:discussion}

This work demonstrates that 
q-dialectical systems—--those that revise beliefs in response to both contradictions and counterexamples--—are strictly more powerful than 
p-dialectical systems, which rely solely on counterexamples. This result clarifies an important theoretical distinction: Reasoning with both contradictions and counterexamples and responding accordingly introduces a form of reasoning that cannot be replicated by counterexamples alone.

In terms of belief dynamics, this distinction maps naturally onto two different operations: excision, the removal of beliefs that lead to contradictions, and replacement, the refinement of beliefs in light of counterexamples. 
Our results thus highlight that both operations are essential for a fully expressive and robust model of rational belief change. Considering the two connections from Section \ref{sec:app total}, we raise the question of how q-dialectical or p-dialectical systems can be used in the two settings of repairing inconsistent knowledge bases, and belief revision. 

These findings deepen our theoretical understanding of how mathematicians or research communities evolve their beliefs. The practice of mathematical inquiry does not rely exclusively on refining conjectures in response to failed examples; it also depends crucially on recognizing when a contradiction signals the need for more fundamental revision. By modeling both forms of reasoning, 
q-dialectical systems more accurately reflect the dual mechanisms driving knowledge development in such domains.

From an artificial intelligence perspective, the implications are similarly significant. An adaptive agent that relies only on counterexample-driven revision may miss critical inconsistencies in its belief state, while one that responds only to contradictions may fail to adjust to new, refining evidence. Our results support the view that both counterexample-based and contradiction-based reasoning are necessary components of intelligent belief management. The integration of both mechanisms, as formalized in 
q-dialectical systems, provides a pathway toward more general and effective approaches to automated reasoning.

\begin{credits}
\subsubsection{\ackname} This work was supported by the the National Science Foundation under Grant DMS-2348792. San Mauro is a member of INDAM-GNSAGA.

\subsubsection{\discintname}
The authors have no competing interests to declare that are
relevant to the content of this article.
\end{credits}

%\printbibliography

\bibliographystyle{splncs04}
%\bibliography{biblio.bib}

\newpage
\appendix

\section{Equivalence of our presentation with the prior definition of q-dialectical systems}\label{Appendix A}

We now give the definition of a q-dialectical system from \cite{dialectical1}, and we then verify that our definition gives an equivalent run.

\subsection{The original definition of q-dialectical systems}

This section verifies that our definition is equivalent to the classical one. Although the argument is somewhat technical, we believe the effort involved in following it only reinforces the value of our more accessible approach.

\begin{definition}\label{def:q-original}
	A q-dialectical system is a quintuple $q=\langle H^\infty, f, f^-, c, c^-\rangle$ such that 
	\begin{itemize}
		\item $H^\infty$ is an enumeration operator so that $H^\infty(\emptyset)\neq \emptyset$, $H^\infty(\{c\})=\nat$, and $H^\infty$ is a Tarskian consequence operator. 
		\item $f$ is a computable permutation of $\nat$. Notationally, we refer to $f(i)$ as $f_i$, representing the $i$th argument in our consideration.
		\item $f^-$ is an acyclic computable function from $\nat$ to $\nat$.
		\item $c^-\notin \ran(f^-)$. 
	\end{itemize}
\end{definition}

We note that, despite the definition referring to $H^\infty$, a run of the system is defined in terms of a computable approximation $H$ to $H^\infty$. An enumeration operator (see \cite[\S 9.7]{rogers1987theory}) $H^\infty$ is a c.e.\ set of pairs $\langle x , F\rangle$, with $x\in \nat$ and $F\in \pfin(\nat)$. We then say $x\in H^\infty(Y)$ if and only if there is some finite $F\subseteq Y$ so that $\langle x , F\rangle\in H$. A computable approximation to $H^\infty$ is a computable sequence $H_s$ so that $H_0=\emptyset$, $H_i\subseteq H_{i+1}$, each $H_i$ is finite, and $\bigcup_i H_i=H^\infty$. Also for each $H_i$, we let $H_i(Y)=\{x\mid \exists F\subseteq Y, \langle x,F\rangle\in H_i\}$.

Now we define the run of a q-dialectical system given a computable approximation $H_i$ to $H^\infty$. Several values are defined by recursion for stages $s$:
$A_s$ (a finite current belief set), $r_s(x)$ is a finite string of numbers for each $x$, viewed as a vertical stack, $p(s)$ (the greatest number $m$ so that $r_s(m)\neq \langle \rangle$)), where $\langle\rangle$ represents the empty sequence, $h(s)$ (a number). There are further derived parameters $\rho_s(x)$ is the last element in the sequence $r_s(x)$ if $r_s(x)$ is non-empty. $\rho_s(x)$ is undefined if $r_s(x)=\langle\rangle$. $L_s(x)=\{\rho_s(y)\mid y<x \wedge r_s(y)\neq \langle\rangle\}$, and for every $i$, define $\chi_s(i)=H_s(L_s(i))$.

A run of $q$ is then defined by the following recursion\footnote{We correct an off-by-one error in the original source.}:

Stage $0$: Define $p(0)=h(0)=0$. $A_0=\emptyset$ 

$$
r_0(x)=\begin{cases}
			\langle f_0 \rangle, & \text{if $x=0$}\\
            \langle\rangle, & \text{otherwise}
		 \end{cases}
$$

Stage $s+1$:  Assume $p(s)=m$. We distinguish cases: 
\begin{enumerate}
	\item For each $z\leq m$, $\{c,c^-\}\cap \chi_s(z)= \emptyset$ (i.e., we see no current need for revision). Let $p(s+1)=m+1$, and define

    $$
    r_{s+1}(x)=\begin{cases}
                    r_s(x), & \text{if $x\leq m$}\\
    			\langle f_{m+1} \rangle, & \text{if $x=m+1$}\\
                \langle\rangle, & \text{if $x>m$}
    		 \end{cases}
    $$

	\item There exists $z\leq m$ such that $c\in \chi_s(z)$, and for all $z'<z$, $c,c^-\notin \chi_s(z)$. In this case, let $p(s+1)=z$ and define

    $$
    r_{s+1}(x)=\begin{cases}
                    r_s(x), & \text{if $x< z-1$}\\
    			\langle f_{z} \rangle, & \text{if $x=z$}\\
                \langle\rangle, & \text{if $x=z-1$ or $x>z$}
    		 \end{cases}
    $$
    
    \item There exists $z\leq m$ such that $c^-\in \chi_s(z)$, for all $z'<z$, $c^-\notin \chi_s(z)$, and $c\notin \chi_s(z)$. Let $f_y=\rho_s(z)$. In this case, let $p(s+1)=z$ and define
    
        $$
        r_{s+1}(x)=\begin{cases}
                        r_s(x), & \text{if $x< z-1$}\\
                        r_s(x)^\smallfrown \langle f^-(f_y)\rangle, & \text{if $x= z-1$}\\
        			\langle f_{z} \rangle, & \text{if $x=z$}\\
                    \langle\rangle, & \text{if $x>z$}
        		 \end{cases}
        $$
\end{enumerate}

Finally, define $h(s+1)=p(s+1)$ if Clause 1 applies, otherwise $h(s+1)=p(s+1)-1$, and let $$A_{s+1}=\bigcup_{i<h(s+1)} \chi_{s+1}(i).$$

The set $A_{s}$ is the \emph{provisional theses} at stage $s$. The set $A^H_q$ defined as
$$A^H_q:= \{f_x \mid \exists t \forall s>t f_x\in A_{s}\}$$ is the set of final theses of $q$.

\begin{definition}
    We say that $q$ is loopless if, for each $y\in \nat$, $\{\rho_s(y)\mid s\in \nat\}$ is finite.
\end{definition}

\subsection{Translating between our definition and the original}

For the sake of clarity in this section, and only in this section, we will refer to the q-dialectical systems as presented in Definition \ref{def:q-ours} as q'-dialectical systems, and the q-dialectical systems as presented in Definition \ref{def:q-original} as q-dialectical systems. As we show the equivalence of these notions, we have no need for this notation elsewhere.

\begin{theorem}
    If $q$ is a loopless q-dialectical system, then there exists a q'-dialectical system $\us$ so that $A^H_q=B_\us$. 
\end{theorem}
\begin{proof}
    Given $q$ a loopless q-dialectical system with an approximation operator $H_i$, we build $\us=(\A,H,r)$ as follows: $\A=\{a_{i} \mid i\in \nat\}$, and we identify the argument $a_i$ with the natural number $f_i$. In a q-dialectical system, $c,c^-\in \nat$, whereas we treat $\bot$ and $\ce$ as logical symbols in q'-dialectical systems. As such, we need to alter $H$ slightly to incorporate these logical symbols.
    Let $E$ be a function from $\pfin(\nat)$ to $\pfin(\nat\cup \{\bot,\ce\})$ as follows: 
    $E(X)=X$ if $c,c^-\notin X$. $E(X)=X\cup \{\bot\}$ if $c\in X,c^-\notin X$. $E(X)=X\cup \{\ce\}$ if $c^-\in X,c\notin X$. Finally $E(X)=X\cup \{\bot,\ce\}$ if $c,c^-\in X$ 
    We let $H(s,X)=E(H_s(X))$, 
    and we let $r(a_i)=a_y$ where $f_y=f^{-}(f_i)$.

    We now check by induction on stages $s$ that the
    %\footnote{We note that the indices are off by 1, because we begin with $\sigma^\us_0$ being the empty string, whereas $r_0(0)=\langle f_0\rangle$, thus corresponding to stage 1 in the run of $\us$.} 
    length of $\sigma^\us_{s}$ is exactly $p(s)$ and 
    \begin{equation}\label{eq:sigma is right}
        \sigma^\us_{s+1}(n) \begin{cases}
                        \rho_s(x), & \text{if $x< p(s), r_s(x)\neq \langle\rangle$}\text{\hfil}\\
        			\ast, & \text{if $x< p(s), r_s(x)= \langle\rangle$}.\text{\hfil}
        		 \end{cases}
    \end{equation}
    Also, $\rho(p(s))=\langle f_{p(s)}\rangle$.

    At $s=0$, we have $p(s)=0$, and $\sigma$ is the empty string, satisfying (\ref{eq:sigma is right}), and $\rho(p(s))=\rho(0)=\langle f_{0}\rangle = \langle f_{p(s)}\rangle$.
    
    We now consider the step of the run and see that we preserve these properties.

    \begin{enumerate}
        \item For each $z\leq p(s)$, $\{c,c^-\}\cap \chi_s(z)= \emptyset$. By inductive hypothesis, $L_s(z)=\ran(\sigma\restrict_z)$ for each $z\leq p(s)$. Thus, we are in Case 2 of Definition \ref{def:q-run-ours}. Then $\sigma_{s+1}^\us=(\sigma_s^\us)^\smallfrown a_{p(s)}$, $p(s+1)=p(s)+1$, and

        $$
    r_{s+1}(x)=\begin{cases}
                    r_s(x), & \text{if $x\leq p(s)$}\\
    			\langle f_{p(s)+1} \rangle, & \text{if $x=p(s)+1$}\\
                \langle\rangle, & \text{if $x>p(s)$}
    		 \end{cases}
    $$
    Since we know $r_s(p(s))=\langle f_{p(s)}\rangle$, we have preserved (\ref{eq:sigma is right}) and that $\rho(p(s+1))=\langle f_{p(s+1)}\rangle$.

        \item There exists $z\leq p(s)$ such that $c\in \chi_s(z)$, and for all $z'<z$, $c,c^-\notin \chi_s(z)$. By induction hypothesis, this means we are in the case of the first bullet point of Case 1 in Definition \ref{def:q-run-ours}. Thus $\sigma^\us_{s+1}=(\sigma^\us_s\restrict_{z-1})^\smallfrown\ast$.
        
        In this case, we have $p(s+1)=z$ and

    $$
    r_{s+1}(x)=\begin{cases}
                    r_s(x), & \text{if $x< z-1$}\\
    			\langle f_{z} \rangle, & \text{if $x=z$}\\
                \langle\rangle, & \text{if $x=z-1$ or $x>z$}.
    		 \end{cases}
    $$

    Note that for $x<z-1$, (\ref{eq:sigma is right}) holds at stage $s+1$ by inductive hypothesis, for $x=z-1$, it holds as in the second case of (\ref{eq:sigma is right}). Further, $\rho(p(s+1))=\langle f_{p(s+1)}\rangle$

    \item There exists $z\leq m$ such that $c^-\in \chi_s(z)$, for all $z'<z$, $c^-\notin \chi_s(z)$, and $c\notin \chi_s(z)$. By inductive hypothesis, we are in the second bullet of Case 1 of Definition \ref{def:q-run-ours}. This means that $\sigma_{s+1}^\us=(\sigma_s^\us\restrict_{z-1})^\smallfrown r(\sigma_s^\us(z-1))$. Let $y$ be so $f_y=\rho_s(z-1)$. Then by inductive hypothesis, $\sigma^\us_s(z-1)=a_y$. Then $r(a_y)=a_k$ where $f_k=f^-(f_y)$.
    
    We have $p(s+1)=z$ and
    
        $$
        r_{s+1}(x)=\begin{cases}
                        r_s(x), & \text{if $x< z-1$}\\
                        r_s(x)^\smallfrown \langle f^-(f_y)\rangle, & \text{if $x= z-1$}\\
        			\langle f_{z} \rangle, & \text{if $x=z$}\\
                    \langle\rangle, & \text{if $x>z$}
        		 \end{cases}
        $$

        Equation (\ref{eq:sigma is right}) holds for $x<z-1$ by inductive hypothesis. It holds for $x=z-1$ as $\sigma^\us_{s+1}(z-1)=r(\sigma_s^\us(z-1))=a_k$ and $\rho_{s+1}(z-1)=f^-(f_y)=f_k$. Finally, $p(s+1)=z$ and $\rho(p(s+1))=\langle f_{p(s+1)}\rangle$.
    \end{enumerate}

    We have now established (\ref{eq:sigma is right}) for every stage $s$. It follows that $a_n=\lim_s \sigma^\us_s$ if and only if $$(\exists t)(\forall s>t)[\rho_s(y)=f_n].$$ We write this as $\lim_s \rho_s(y)=f_n$.

    Then $B_\us=\{a_n\mid \exists y \lim_s \rho_s(y)=f_n\}$. It follows that $B_\us\subseteq A_q$ (identifying $a_n$ with $f_n$).

    By \cite[3.14]{dialectical1}, $A_q\subseteq \{f_n \mid f_n=\lim_s \rho_s(n)\}$, which is clearly a subset of $B_\us$. Thus $A_q=B_\us$.
\end{proof}

\begin{theorem}
    If $\us$ is a loopless q'-dialectical system, then there exists a q-dialectical system $q$ so that $A^H_q=B_\us$. 
\end{theorem}
\begin{proof}
    Given a loopless q'-dialectical system $\us$, we build a q-dialectical system $q=\langle H^\infty,f,f^-,c,c^-\rangle$. In $\us$, $\bot$ and $\ce$ are logical symbols, but in $q$, we must make them arguments. To do so, we let $0$ represent the argument $\bot$, $1$ represent the argument $\ce$, and $n+2$ represent the argument $a_n$. We then let $f$ be the identity on $\nat$ and we identify $f_i=i$ with $a_i$. We define $f^-(f_0)=f_1,f^-(f_1)=f_0$, and $f^-(f_{n+2})=f_{k+2}$ where $r(a_n)=a_k$. 

    We define a computable function $M$ so that $M(s)$ is greater than $r^t(x)$ for each $x\leq s$ and $t\leq s$. In particular, $M$ is a number guaranteed to be larger than any $k$ so that $a_k$ can appeared in the run of $\us$ by stage $s$.
    We let $H_{s+5}=\{\langle x, F \rangle \mid a_{x-2}\in H(s,\{a_{y-2}\mid y\in F\})\wedge \max(F)\leq M(s)\}\cup \{\langle 0,F\rangle \mid \bot\in H(s,F)\wedge \max(F)\leq M(s)\}\cup \{\langle 1,F\rangle \mid \ce\in H(s,F)\wedge \max(F)\leq M(s)\}$ and note that this defines $H^\infty$ which is a Tarskian consequence operator. The role of $M$ is to ensure that $H_s$ is a finite set for each $s$, since we did not assume finiteness of an \aco. This makes no difference in the run of either $\us$ or $q$, since any set $F$ which we can possibly consider in a run of either $\us$ or $q$ has the property that $H_{s+5}(F)=H(s,F)$.
    
    We let $c=0$ and $c^-=1$. This defines our q-dialectical system $q$. We aim to show that $a_i\in B_\us$ if and only if $f_{i+2}\in A_q^H$. Since we identify $a_n$ with $f_{n+2}$, this means $A_q^H=B_\us$.

    The first few steps of the run of $q$ are determined by the fact that $c=0$ and $c^-=1$. At the first few stages:

    $$
    r_0(x)=\begin{cases}
			\langle f_0 \rangle, & \text{if $x=0$}\\
            \langle\rangle, & \text{otherwise}
		 \end{cases}
    $$

    $$
    r_1(x)=\begin{cases}
    			\langle f_x \rangle, & \text{if $x\leq 1$}\\
                \langle\rangle, & \text{otherwise}
    		 \end{cases}
    $$

    $$
    r_2(x)=\begin{cases}
    			\langle f_1 \rangle, & \text{if $x=1$}\\
                \langle\rangle, & \text{otherwise}
    		 \end{cases}
    $$

    $$
    r_3(x)=\begin{cases}
    			\langle f_x \rangle, & \text{if $x=1,2$}\\
                \langle\rangle, & \text{otherwise}
    		 \end{cases}
    $$

    $$
    r_4(x)=\begin{cases}
    			\langle f_1,f_0 \rangle, & \text{if $x=1$}\\
                \langle f_2 \rangle, & \text{if $x=2$}\\
                \langle\rangle, & \text{otherwise}
    		 \end{cases}
    $$
    $$
    r_5(x)=\begin{cases}
                \langle f_2 \rangle, & \text{if $x=2$}\\
                \langle\rangle, & \text{otherwise}
    		 \end{cases}
    $$

    We note that at stage 5, we satisfy the following equation. Exactly as in the detailed induction of the previous theorem, we will maintain the following equation for each $s\geq 5$: 

    \begin{equation}\label{eq:sigma is right - shifted}
        \sigma^\us_{s-5}(n) = \begin{cases}
                        \rho_s(n+2), & \text{if $n+2< p(s), r_s(n+2)\neq \langle\rangle$}\text{\hfil}\\
        			\ast, & \text{if $n+2< p(s), r_s(n+2)= \langle\rangle$}\text{\hfil}
        		 \end{cases}
    \end{equation}

    Note that if $\rho_s=f_{k}$, we are identifying $f_k$ with the argument $a_{k-2}$. In this sense, we say $\sigma^\us_{s-5}(n)=\rho_s(n+2)$, i.e. $\sigma^\us_{s-4}(n)=a_{k-2}$ where $\rho_s(n+2)=f_k$.
    %Further, $\rho(p(s))=\langle f_{p(s)}\rangle$.

    It follows as in the previous theorem that $A_{q}=B_\us$, i.e. $a_n\in B_\us$ if and only if $f_{n+2}\in A_q$.

    We note that it follows that if $\us$ is loopless, then $B_\us=H^\infty(B_\us)$. This is immediate from the fact that $H^\infty(A_q)=A_q$.
\end{proof}

Next we consider the case with loops, where the sets $B_\us$ for q'-dialectical systems are not the same as the sets $A_q$ for q-dialectical systems. This is a cosmetic difference, since the set $A_q$ is defined as being closed under $H^\infty$, whereas $B_\us$ is not.

We use a characterization \cite[Theorem 3.13]{dialectical1} of the sets $A_q$ for q-dialectical systems with loops. A co-infinite c.e.\ set is non-simple if its complement contains an infinite c.e.\ subset. As every infinite c.e.\ set contains an infinite computable subset, this is equivalent to saying that a co-infinite c.e.\ set is non-simple if its complement contains an infinite computable subset (see \cite[\S 8.1]{rogers1987theory} for more on simple c.e.\ sets). 

\begin{theorem}
    The collections of sets $A_q$ for $q$ a q-dialectical system with a loop are exactly the sets $H^\infty(B_\us)$ for q'-dialectical systems $\us$. Both coincide with the collection of co-infinite non-simple c.e.\ sets.
\end{theorem}
\begin{proof}
    The fact that the collection of sets $A_q$ for q-dialectical systems is exactly the collection of co-infinite non-simple c.e.\ sets is \cite[Theorem 3.13]{dialectical1}.

    Fix a co-infinite non-simple c.e.\ set $X$. Let $X=\bigcup_s X_s$ be a computable approximation to $X$. This means that if $X$ is the range of the computable function $f$, then $X_s$ is the range of $f\restrict_{[0,s]}$. Then $X=\bigcup_s X_s$. 
    
    Let $B$ be an infinite computable subset of $\nat\smallsetminus X$. Let $(b_i)_{i\in \omega}$ be a computable enumeration of $B$.
    Let $\A=\{a_i\mid i\in \nat\}$ be any enumeration of $\nat$ where $a_0=b_0$. We define $r(b_i)=b_{i+1}$, and for $x\notin B$, $r(x)=y$ where $y$ is the least element of $\nat\smallsetminus B$ which is $>x$. Let $H(s,Y)=X_s$ if $Y\cap B=\emptyset$ and $H(s,Y)=X_s\cup \{\ce\}$ if $Y\cap B\neq\emptyset$. It is now straightforward to check that $B_\us=\emptyset$, as the element $a_0=b_0$ will keep getting replaced by elements of $B$. Thus $H^\infty(B_\us)=H^\infty(\emptyset)=X$. 

    Let $\us$ be a q'-dialectical system with a loop. Then $B_\us$ is finite. Let $X=H^\infty(B_\us)$. Then $X$ is c.e.\ since $H^\infty$ is computably approximated by $H$. Let $n$ be least so that $\lim_s\sigma^\us_s(n)$ is undefined. Then the set $\{r^k(a_n)\mid k\in \nat\}$ is in $\nat\smallsetminus X$. Since $r$ is acyclic, this shows that $X$ is co-infinite and non-simple. 
\end{proof}

Putting these results together, we have established the following theorem:
\begin{theorem}\label{thm:equivalence}
    The sets $A_q$ for $q$ a q-dialectical system are exactly the sets $H^\infty(B_\us)$ for $\us$ a q'-dialectical system. In the loopless case, then $H^\infty(B_\us)=B_\us$.
\end{theorem}

We note that since p-dialectical and (d-)dialectical systems are special cases of q-dialectical systems (where we do not use $\bot$ or $\ce$, respectively), the analog of Theorem \ref{thm:equivalence} follows for p-dialectical and (d-)dialectical systems as well.

\section{Details for the verification of Theorem \ref{thm:Main}}\label{sec:appx-verification}

\lemmahandsoff*
\begin{proof}
    We prove this by induction on the strategies. For the strategy $\strat_0$ this is trivially true. We suppose it is true for the strategy $\strat_j$, and we must prove that it is true for the strategy $\strat_{j+1}$. 
    In particular, $N$ is chosen large, so that it is not part of any counterexample or contradiction placed by a higher-priority strategy, i.e., if a higher-priority strategy put $\ce$ into $H(s,F)$ for some set $F$, then $N$ is larger than any $k$ for $a_k\in F$. Thus this consequence will ensure that $F\not\subseteq \ran(\sigma^\us_t)$ for all later stages $t$, but will not remove $a_K$ for any $K>N$ from $\ran(\sigma^\us_t)$. 
    We consider the possibilities based on the outcome of the strategy $\strat_j$:

    If $\strat_j$ gets stuck in Step 2 or PO2, then it causes no contradictions or counterexamples at all. Then $\strat_{j+1}$ chooses its parameter $N$, and we have no counterexamples or contradictions involving any $a_k$ with $k\in [N_j,N)$, so we still have $B_\us\cap \{a_i \mid i<N\} = \{a_i \mid i<N\}\smallsetminus \bigcup_{k<j} Z_k=\{a_i \mid i<N\}\smallsetminus \bigcup_{k< j+1} Z_k$.

    If $\strat_j$ gets stuck in Step 5, then $\strat_j$ has placed one counterexample or contradiction, either putting $\bot$ or $\ce$ in $H(s,S\cup \{a_{N_j}\})$. Since, by inductive hypothesis, $S\subseteq B_\us$, it follows that the effect of this one counterexample or contradiction is that $a_{N_j}\notin B_\us$. Further, the strategy $\strat_{j+1}$ is last activated when $\strat_j$ enters Step 5. Though lower-priority strategies may have acted previously, they all put $a_{N_j}$ into every counterexample or contradiction they put in $H$, so they have no effect on $B_\us$, since at every stage $t$ after $s$, we have that $a_{N_j}\notin \ran((\sigma^\us_t)^-)$. It follows that no \axiom $a_K$ with $K\in [N_j,N)$ will be either replaced or excised after stage $s$.

    If $\strat_j$ gets stuck in Step 7, then $\strat_j$ has placed 2 counterexamples or contradictions: One puts either $\bot$ or $\ce$ into $H(t,S\cup \{a_{N_j}\})$. The other puts $\bot$ into $H(t,S\cup \{a_{N_j+1},a_{N_j+2}\})$. This has the immediate effect of having either excised or replaced $a_{N_j}$ and having either $a_{N_j+1}$ or $a_{N_j+2}$ excised (depending on order). These are the only additional consequences placed in $H$ by $\strat_j$. Again, all lower-priority strategies which placed counterexamples or contradictions before $\strat_j$ entered Step 7 did so with $a_{N_j+1},a_{N_j+2}$ both in the hypotheses. It follows that they have no effect after the stage where we put $\bot$ into $H(t,S\cup \{a_{N_j+1},a_{N_j+2}\})$. Thus no \axiom $a_K$ with $K\in [N_{j}+3,N)$ will be either replaced or excised after stage $s$. We note that among $a_{N_j+1}$ and $a_{N_j+2}$, only the one which appears first in $\sigma^\us_t$ will remain, showing that $Z_j$ is precisely the \axioms removed from $B_\us$ among $\{a_{N},a_{N+1},a_{N+2}\}$.

    If $\strat_j$ gets completed in Step 8: Then $B_\us\neq B_{\them_j}$ has placed 3 counterexamples or contradictions in $H$: One puts either $\bot$ or $\ce$ into $H(t,S\cup \{a_{N_j}\})$. Another puts $\bot$ into $H(t,S\cup \{a_{N_j+1},a_{N_j+2}\}$. The third puts $\bot$ into $H(t',S\cup \{a_{J}\})$ for one $J\in \{a_{N_j+1},a_{N_j+2}\}$. In this case, the immediate result is that $\sigma^\us_{t'}=\sigma^\us_{s}\vert N^\smallfrown \gamma$ where $\gamma$ is a string of length 3 comprised of entries which are either $\ast$ or equal to the $a_I$ with $a_J\neq a_I\in \{a_{N_j+1},a_{N_j+2}\} $. Note that every counterexample or contradicion which was placed in $H$ previously by lower-priority requirements had either $a_{N}$ in the hypotheses (if placed before Step 4 of $B_\us\neq B_{\them_j}$) or had $a_J$ in the hypotheses (if placed during Step 7). In particular, none of these will have any effect at future stages where we have put $a_{I}$ into $\ran(\sigma^\us_{t'})$. Again, we see that every $K\in [N_{j}+3,N)$ will never be removed after stage $s$ and that $Z_i$ is precisely the \axioms among $\{a_{N_j},a_{N_j+1}, a_{N_j+2}\}$ which are removed from $\lim_s \sigma^\us_s$.   
\end{proof}

\lemmaEsAreRight*
\begin{proof}
    We apply Lemma \ref{lem:hands off} to the strategy $\strat_{i+1}$. Note that $M$ is mentioned by stage $s$ and, when \texttt{PredictOrder}$(\them_i,s,a_N,a_{N+1},a_{N+2})$ returns a value at stage $s$, the strategy $\strat_i$ acts via Step 4. This deactivates all lower-priority strategies. In particular, when the strategy $\strat_{i+1}$ is finally activated, it chooses a parameter $N>M$. Thus we can apply Lemma \ref{lem:hands off} to see that $B_\us\cap \{a_i \mid i<N\} = \{a_i \mid i<N\}\smallsetminus \bigcup_{j\leq i} Z_j$. In particular, since $Z_i\subseteq \{a_N,a_{N+1},a_{N+2}\}$ and $a_N\in Z_i$, we get that $B_\us\cap (\{a_j \mid j\leq M\}\smallsetminus \{a_{N+1},a_{N+2}\})= \{a_j \mid j\leq M\}\smallsetminus\bigcup_{j< i} Z_j=\{a_j \mid j\leq M\}\smallsetminus (\{a_{N+1},a_{N+2}\}\cup E)$.
\end{proof}

\lemmaPredictOrderIsRight*
\begin{proof}
    % We first note that if $B_\us= B_{\them_i}$, then in particular, $\them_i$ is a p-dialectical system and $r^{\them_i}$ must be a total function. In particular, since $r^{\them_i}$ is also acyclic, we will at some stage see numbers $o_j$ so that $(r^{\them_i})^{\circ o_j} (g_j)\notin E$, since $E$ was defined to be a finite set. Thus the function PredictOrder$(\them_i,s,a_N,a_{N+1},a_{N+2})$ will return some string $\rho$. At that point, we will either act by adding $\ce$ or $\bot$ to $H(s',S\cup \{a_N\})$. This invalidates any \axiom placed into $H$ by lower-priority requirements. Note that all lower-priority requirements which become active \emph{after} this stage choose their $N$ larger than any element of $\ran(\tau)$, since each of these numbers have been mentioned.

    We argue that either $B_\us\neq B_{\them_i}$ or at some later stage $t$, we have $\rho\preceq \sigma^{\them_i}_t$. Suppose towards a contradiction that we never see $\rho\preceq \sigma^{\them_i}_t$ for any later stage $t$. We prove by induction on $e\leq \vert \rho\vert$ that at all sufficiently large stages, we will see $\rho\restrict_e\preceq \sigma^{\them_i}_t$. In particular, for $e=\vert \rho \vert$, this will be our desired contradiction. Note that this is trivially true for $e=0$.
    
    For the induction, we will proceed to show that $\rho\restrict_{e+1}\preceq \sigma^{\them_i}_t$ for all sufficiently large $t$, given that we know $\rho\restrict_e\preceq \sigma^{\them_i}_t$ for all sufficiently large $t$. Fix $e<n$ and $t_0$ large enough that $\rho\restrict_e \preceq \sigma_t^{\them_i}$ for all $t\geq t_0$. Now we argue that $\sigma^{\them_i}_{t'}$ will change with $t'$ until we get to $(r^{\them_i})^{e_j}(g_e)$. If $\sigma^{\them_i}_{t'}(e)\in E$ and is not replaced, then $\sigma^{\them_i}_{t'}(e)\in B_{\them_i}\smallsetminus B_\us$ by Lemma \ref{lem: Es are right}. Thus $\sigma^{\them_i}_{t'}(e)$ must be replaced until it is no longer in $E$. But once it finds an \axiom $\notin E$, this must be the limiting value of $\sigma^{\them_i}_{t'}(e)$. This is because a change to this would require $\ce$ to enter $H^{\them_i}(s',\sigma^{\them_i}_{t'}\restriction_{e+1})$. In particular, a subset of $\{a_j \mid j\leq M\}\smallsetminus (\{a_{N+1},a_{N+2}\}\cup E)$ gives a counterexample in $H^{\them_i}$. But by Lemma \ref{lem: Es are right}, this subset is a subset of $B_\us$, which would then not be a subset of $B_{\them_i}$, showing that $B_\us\neq B_{\them_i}$. This shows that the limiting value of $\sigma^{\them_i}_{t'}(e)$ is $\rho(e)$, thus $\rho\restrict_{e+1}\preceq \sigma^{\them_i}_t$ for all sufficiently large $t$, completing the induction.
    % Inducting as such brings us to the value $e$ so that $\tau(e)=a_{N+1}$ or $a_{N+2}$. This is the $\tau$ that PredictOrder$(\them_i,s,a_N,a_{N+1},a_{N+2})$ returns, and we have argued that either the requirement is satisfied or at some point we will see $\tau\preceq \sigma_t^{\them_i}$.

    By Lemma \ref{lem: Es are right}, $\ran(\rho\restriction_{\vert \rho\vert -1})$ is a subset of $B_\us$. Thus once we have $\rho\restriction_{\vert \rho\vert -1}\preceq \sigma^{\them_i}_t$, we either have $\rho^-\preceq \lim_s \sigma_s^{\them_i}$ or some subset $F$ of $\ran(\rho\restriction_{\vert \rho\vert -1})$ must eventually have $\ce\in H(t,F)$, but this would ensure that $F\not\subseteq B_{\them_i}$ despite $F$ being a subset of $B_\us$.
\end{proof}

\lemmaStopInjury*
\begin{proof}
    This is proved by induction on $i$. Starting at stage $1$, $\strat_0$ is activated and can never be deactivated since there is no higher-priority strategy.

    Let $s$ be a stage after which all requirements $\strat_i$ for $i<K$ are always activated. Then no strategy $\strat_j$ with $j<K-1$ can ever act again, or it would deactivate the strategy $\strat_{K-1}$. Then the strategy $\strat_{K-1}$ can act at most finitely many times after stage $s$ (as it moves through its finitely many steps). After these finitely many actions of $\strat_{K-1}$, the strategy $\strat_K$ will be activated and can never be deactivated again.
\end{proof}

\lemmaWeWin*
\begin{proof}
    We need to show that the strategy $\strat_i$ succeeds. Let $s$ be the least stage after which the strategy $\strat_i$ is never deactivated. 
    % If the strategy $B_\us\neq B_{\them_i}$ gets stuck in Step 2, let $Z_i=\emptyset$. If it gets stuck in Step 5, let $Z_i=\{a_N\}$. If the strategy $B_\us\neq B_{\them_i}$ gets stuck in Step 7, let $Z_i=\{a_N,a_i\}$. If the strategy $B_\us\neq B_{\them_i}$ gets stuck in Step 2, let $Z_i=\{a_N\}$. If the strategy goes throug Step 8, let $Z_i=\{a_N,a_{j}\}$. The set $Z_i$ represents the elements of $\A$ which the strategy $B_\us\neq B_{\them_i}$ removes from $B_\us$. 
    We now argue in each possible outcome of $\strat_i$ that $B_\us\neq B_{\them_i}$.

    If the strategy gets stuck in the WAIT of Step 2, then $a_N,a_{N+1},a_{N+2}$ are not all in the domain of $\them_i$, so certainly are not in $B_{\them_i}$, but $a_N,a_{N+1},a_{N+2}\in B_\us$ by Lemma \ref{lem:hands off}.

    If the strategy gets stuck in the WAIT of Step 5, then Lemma \ref{lem:PredictOrder is right} shows that $B_\us\neq B_{\them_i}$.

    If the strategy gets stuck in Step 7, then $a_i\in B_{\them_i}$ since $\rho\preceq \lim_s \sigma_s^{\them_i}$. But $a_i\notin B_\us$ by Lemma \ref{lem:hands off}.

    If the strategy gets stuck in Step PO2, then $r^{\Lambda_i}$ is either partial or cyclic. In either case, this implies that $\them_i$ is not a p-dialectical system at all. 

     Finally, we consider the case where the strategy completes Step 8. By Lemma \ref{lem:hands off}, $a_i\in B_\us$, but in Step 7, we saw $\ce\in H^{\them_i}(t',\rho)$. But since $\rho^-\preceq \lim_s \sigma^{\them_i}_s$, this implies that $a_i\notin B_{\them_i}$.
    % $S\subseteq B_\us$, which implies that $a_j\notin B_\us$. Since the only \axiom involving $a_i$ without any $a_K$ for $K>{N+2}$ is the \axiom $\bot\in H(t, S\cup \{a_i,a_j\})$, we see that $a_{i}\in B_\us$. But in Step 7, we saw $\ce\in H^{\them_i}(t',\rho)$. But since $\rho^-\preceq \lim_s \sigma^{\them_i}_s$, this implies that $a_i\notin B_{\them_i}$
\end{proof}

\end{document}